\documentclass{llncs}
\usepackage{times}
\usepackage{xcolor}
\usepackage{amsmath}
\usepackage{amsfonts}
\usepackage{graphicx}
\usepackage{subfig}

\usepackage[colorlinks=true,linkcolor=blue]{hyperref}

\newcommand{\norm}[1]{\|#1\|}
\newcommand{\xplus}{\bs{X}_{t+1}}

\newcommand{\x}[1][t]{\bs{X}_{#1}}

\newcommand{\suite}[1]{({#1}_t)_{t \in \N}}

\newcommand{\sig}[1][t]{\sigma_{#1}}
\newcommand{\sigplus}{\sigma_{t+1}}

\newcommand{\yi}{\bs{Y}_{t, i}}
\newcommand{\y}[1]{\bs{Y}_{t, #1}}
\newcommand{\pplus}{\bs{p}_{t+1}}
\newcommand{\p}[1][t]{\bs{p}_{#1}}

\newcommand{\pdim}[1][i]{\left[\pplus\right]_{#1}}
\newcommand{\inv}[1]{\frac{1}{#1}}

\newcommand{\randomstep}[1][i]{\bs{\xi}_{t,#1}}
\newcommand{\chosenstep}[1][t]{\bs{\xi}_{#1}^\star}
\newcommand{\chosenstepdistribution}{\bs{\xi}}
\newcommand{\stepsizechangelong}[1][\bs{p}_{t+1}]{\exp\left( \frac{c}{2d_\sigma}\left( \frac{\norm{#1}^2}{n} - 1 \right)\right)}

\newcommand{\hsp}{\hspace{1mm}}
\newcommand{\firstdim}[1]{\left[#1\right]_1}
\newcommand{\fchosenstep}[1][t]{\firstdim{\chosenstep[#1]}}

\newcommand{\pchain}{\bs{\mathcal{P}}}

\newcommand{\E}{\mathbb{E}}
\newcommand{\R}{\mathbb{R}}
\newcommand{\NL}{\bs{\mathcal{N}}(\bs{0}, {I_n)}}
\newcommand{\Nlambda}[1][1]{\mathcal{N}_{#1:\lambda}}
\newcommand{\Nmin}[1][\lambda]{\mathcal{N}_{1:{#1}}}
\newcommand{\N}{\mathbb{N}}
\newcommand{\NNN}{\mathcal{N}}

\newcommand{\orderstatdis}[2][1]{\Psi_{#1:#2}}

\newcommand{\bs}[1]{\boldsymbol{#1}}

\DeclareMathOperator{\diff}{d}
\DeclareMathOperator{\argmin}{argmin}
\DeclareMathOperator{\Var}{Var}

\begin{document}

\title{Cumulative Step-size Adaptation on Linear Functions}
\author{Alexandre Chotard\inst{1}, Anne Auger\inst{1} \and
Nikolaus Hansen\inst{1}}
\authorrunning{Alexandre Chotard et al.} 
\tocauthor{Alexandre Chotard, Anne Auger and
Nikolaus Hansen}
\institute{TAO team, INRIA Saclay-Ile-de-France, LRI, Paris-Sud University,
France\\
\email{firstname.lastname@lri.fr}}

\maketitle       

\begin{abstract}
The CSA-ES is an Evolution Strategy with Cumulative Step size Adaptation, where the step size is adapted measuring the length of a so-called cumulative path. The cumulative path is a combination of the previous steps realized by the algorithm, where the importance of each step decreases with time.
This article studies the CSA-ES on composites of strictly increasing functions with affine linear functions through the investigation of its underlying Markov chains. Rigorous results on the change and the variation of the step size are derived with and without cumulation. The step-size diverges geometrically fast in most cases.  Furthermore, the influence of the cumulation parameter is studied.
\end{abstract}


\section{Introduction}
\newcommand{\mc}[2]{\newcommand{#1}{\ensuremath{#2}}}
\mc{\X}{\bs{X}}

Evolution strategies (ESs) are continuous stochastic optimization algorithms searching for the minimum of a real valued function $f:\R^n\to\R$. In the ($1,\lambda$)-ES, in each iteration, $\lambda$ new children are generated from a single parent point $\X\in\R^n$ by adding a random Gaussian vector to the parent,

\[
	\bs{X} \in \R^n \mapsto \bs{X} + \sigma \bs{\mathcal{N}} (\bs{0}, \bs{C}) 
	\enspace.
\]
Here, $\sigma\in \R_{+}^{*}$ is called step-size and $\bs{C}$ is a covariance matrix. The best of the $\lambda$ children, i.e.\ the one with the lowest $f$-value, becomes the parent of the next iteration. 
To achieve reasonably fast convergence, step size and covariance matrix have to be adapted throughout the iterations of the algorithm. In this paper, $\bs{C}$ is the identity and we investigate the so-called Cumulative Step-size Adaptation (CSA), which is used to adapt the step-size in the Covariance Matrix Adaptation Evolution Strategy (CMA-ES) \cite{ostermeier1994nonlocal,cmaesbirth}. In CSA, a cumulative path is introduced, which is a combination of all steps the algorithm has made, where the importance of a step decreases exponentially with time. Arnold and Beyer studied the behavior of CSA on sphere, cigar and ridge functions \cite{arnold2004performance,arnold2010behaviour,arnold2006cumulative,arnold2008evolution} and on dynamical optimization problems where the optimum moves randomly \cite{arnold2002random} or linearly \cite{arnold2006optimum}. Arnold also studied the behaviour of a ($1,\lambda$)-ES on linear functions with linear constraint \cite{arnold2011behaviour}. 

In this paper, we study the behaviour of the $(1,\lambda)$-CSA-ES on composites of strictly increasing functions with affine linear functions, e.g.\ $f:\vec x\mapsto \exp(x_2 - 2)$. Because the CSA-ES is invariant under translation, under change of an orthonormal basis (rotation and reflection), and under strictly increasing transformations of the $f$-value, we investigate, w.l.o.g., $f: \vec x \mapsto x_1$. Linear functions model the situation when the current parent is far (here infinitely far) from the optimum of a smooth function. To be far from the optimum means that the distance to the optimum is large, \emph{relative to the step-size $\sigma$}. This situation is undesirable and threatens premature convergence. 
The situation should be handled well, by increasing step widths, by any search algorithm (and is not handled well by the $(1,2)$-$\sigma$SA-ES \cite{hansen2006ecj}). Solving linear functions is also very useful to prove convergence independently of the initial state on more general function classes.

In Section~\ref{sec:CSA} we introduce the $(1, \lambda)$-CSA-ES, and some of its characteristics on linear functions. In Sections~\ref{sec:withoutcumulation} and \ref{sec:withcumulation} we study $\ln(\sigma_t)$ without and with cumulation, respectively. Section \ref{sec:var} presents an analysis of the variance of the logarithm of the step-size and in Section~\ref{sec:conclusion} we summarize our results.

\mc{\xx}{\bs{x}}
\mc{\NN}{\mathcal{N}(0,1)}
\paragraph*{Notations}
In this paper, we denote $t$ the iteration or time index, $n$ the search space dimension, \NN\ a standard normal distribution, i.e. a normal distribution with mean zero and standard deviation 1. The multivariate normal distribution with mean vector zero and covariance matrix identity will be denoted $\NL$, the $i^{\rm th}$ order statistic of $\lambda$ standard normal distributions $\Nlambda[i]$, and $\Psi_{i:\lambda}$ its distribution. If $\bs{x} = \left(x_1, \cdots, x_n \right) \in \R^n$ is a vector, then $\left[x\right]_i$ will be its value on the $i^{th}$ dimension, that is $\left[x\right]_i = x_i$. A random variable $\bs{X}$ distributed according to a law $\mathcal{L}$ will be denoted $\bs{X} \sim \mathcal{L}$.

\section{The $(1, \lambda)$-CSA-ES} \label{sec:CSA}

We denote with $\x$ the parent at the $t^{th}$ iteration.
From the parent point $\x$, $\lambda$ children are generated: $\yi = \x + \sig[t] \randomstep$ with $i \in [[1, \lambda]]$, and $\randomstep \sim \NL, \hsp (\randomstep)_{i \in [[1, \lambda]]} \hsp$ i.i.d.  Due to the $(1,\lambda)$ selection scheme, from these children, the one minimizing the function $f$ is selected: $\xplus = \argmin \{f(\bs{Y}), \bs{Y} \in \{{\y{1}, ..., \y{\lambda}\}}\}$.
%%
%\begin{equation} \label{eq:selection1}
%	\xplus = \underset{\bs{Y} \in \{{\y{1}, ..., \y{\lambda}\}}}{\argmin} (f(\bs{Y}))
%\end{equation}	
%%
This latter equation implicitly defines the random variable $\chosenstep$ as
\begin{equation}
\xplus = \x + \sig[t] \chosenstep \label{eq:selection2}
\enspace.
\end{equation}
In order to adapt the step-size, the cumulative path is defined as
\begin{equation}
\pplus = (1-c)\p + \sqrt{c(2-c)}\, \chosenstep \label{eq:chemin}
\end{equation}
with $0<c\leq 1$. The constant $1/c$ represents the life span of the information contained in $\p$, as after $1/c$ generations $\p$ is multiplied by a factor that approaches $1/e \approx 0.37$ for $c\to0$ from below (indeed $(1-c)^{1/c} \leq \exp(-1)$). The typical value for $c$ is between $1/\sqrt{n}$ and $1/n$. We will consider that $\bs{p}_0 \sim \NL$ as it makes the algorithm easier to analyze.

The normalization constant $\sqrt{c(2-c)}$ in front of $\chosenstep$ in Eq.~\eqref{eq:chemin} is  chosen so that under random selection and if $\p$ is distributed according to $\NL$ then also $\pplus$ follows $\NL$. Hence the length of the path can be compared to the expected length of $\| \NL \|$ representing the expected length under random selection. 

The step-size update rule increases the step-size if the length of the path is larger than the length under random selection and decreases it if the length is shorter than under random selection: 
%\note{donner des valeurs pour $d_{\sigma}$}
$$
\sigplus = \sig[t] \exp\left({\frac{c}{d_{\sigma}}\left(\frac{\|\pplus\|}{E(\|\NL\|)} - 1\right)}\right) %\enspace.
$$
where the damping parameter $d_{\sigma}$ determines how much the step-size can change and is set to $d_{\sigma}=1$. 
A simplification of the update considers the squared length of the path~\cite{arnold2002random}:
\begin{equation}
\sigplus = \sig[t] \stepsizechangelong. \label{eq:pas}
\end{equation}
This rule is easier to analyse and we will use it throughout the paper.

\paragraph{Preliminary results on linear functions.}

Selection on the linear function, $f(\xx)=[\xx]_{1}$, is determined by  $\firstdim{\x} + \sig\firstdim{\chosenstep}  \leq  \firstdim{\x} + \sig\firstdim{\randomstep} $ for all $i$ which is equivalent to $ \firstdim{\chosenstep}  \leq  \firstdim{\randomstep} $ for all $i$ where by definition $\firstdim{\randomstep}$ is distributed according to $\NN$. Therefore the first coordinate of the selected step is distributed according to $\Nmin$ and all others coordinates are distributed according to $\NN$, i.e. selection does not bias the distribution along the coordinates $2,\ldots,n$. Overall we have the following result.
\begin{lemma}\label{lem:selectedstep}
On the linear function $f(\xx)=x_{1}$, the selected steps $(\chosenstep)_{t\in \N}$ of the $(1,\lambda)$-ES are i.i.d. and distributed according to the vector $\chosenstepdistribution:=(\Nmin,\NNN_{2},\ldots,\NNN_{n})$ where $\NNN_{i} \sim \NN$ for $i \geq 2$.
\end{lemma}

Because the selected steps $\chosenstep$ are i.i.d.\ the path defined in Eq.~\ref{eq:chemin} is an autonomous Markov chain, that we will denote $\pchain = (\p)_{t\in \N}$. Note that if the distribution of the selected step depended on $(\x,\sigma_{t})$ as it is generally the case on non-linear functions, then the path alone would not be a Markov Chain, however $(\x,\sigma_{t},\p)$ would be an autonomous Markov Chain. In order to study whether the $(1,\lambda)$-CSA-ES diverges geometrically, we investigate the log of the step-size change, whose formula can be immediately deduced from Eq.~\ref{eq:pas}:
\begin{equation}\label{eq:stepsizechange}
 \ln\left( \frac{\sigma_{t+1}}{\sig} \right)   =  \frac{c}{2d_\sigma}  
\left( \frac{\norm{\pplus}^2}{n} - 1 \right)  
% \enspace.
\end{equation}
By summing up this equation from $0$ to $t-1$ we obtain
\begin{equation} \label{eq:sigcumul}
\inv{t}\ln\left(\frac{\sigma_{t}}{\sigma_0} \right) = \frac{c}{2d_\sigma}  \left( \inv{t}\sum_{k=1}^{t} \frac{\norm{\p[k]}^2}{n} - 1 \right) \enspace.
\end{equation}
We are interested to know whether $\inv{t}\ln ({\sigma_{t}/ \sigma_0} )$ converges to a constant. In case this constant is positive this will prove that the $(1,\lambda)$-CSA-ES diverges geometrically. We recognize thanks to \eqref{eq:sigcumul} that this quantity is equal to the sum of $t$ terms divided by $t$ that suggests the use of the law of large numbers to prove convergence of \eqref{eq:sigcumul}. We will start by investigating the case without cumulation $c=1$ (Section~\ref{sec:withoutcumulation}) and then the case with cumulation (Section~\ref{sec:withcumulation}).

\section{Divergence rate of $(1, \lambda)$-CSA-ES without cumulation} 

\label{sec:withoutcumulation}
In this section we study the $(1,\lambda)$-CSA-ES without cumulation, i.e. $c=1$. In this case, the path always equals to the selected step, i.e.\ for all $t$, we have $\pplus = \chosenstep$. We have proven in Lemma~\ref{lem:selectedstep} that $\chosenstep$ are i.i.d.\ according to $\chosenstepdistribution$. This allows us to use the standard law of large numbers to find the limit of $\frac1t \ln(\sigma_{t}/\sigma_{0})$ as well as compute the expected log-step-size change.

\mc{\Deltasigma}{\Delta_\sigma}
\begin{proposition} \label{pr:sigrate}
Let $\Deltasigma := \frac{1}{2d_{\sigma}n} \left( \E \left( \Nlambda^2 \right) - 1 \right)$. On linear functions, the $(1,\lambda)$-CSA-ES without cumulation satisfies (i) almost surely
$ \lim_{t \to \infty}	\inv{t} \ln\left({\sig}/{\sig[0]}\right) = \Deltasigma 
%\frac{1}{2d_{\sigma}n} \left( \E \left( \Nlambda^2 \right) - 1 \right)
$,
and (ii) for all $t\in\N$, 
$\E (\ln({\sigma_{t+1}}/{\sig})) = \Deltasigma
%\frac{1}{2d_{\sigma}n} \left( \E \left( \Nlambda^2 \right) - 1 \right)
$.
\end{proposition}
\begin{proof} 
We have identified in Lemma~\ref{lem:selectedstep} that the first coordinate of $\chosenstep$ is distributed according to $\Nmin$ and the other coordinates according to $\mathcal{N}(0,1)$, hence $\E\left(\|\chosenstep\|^2\right)=\E\left( {\firstdim{\chosenstep}}^2  \right) + \sum_{i=2}^n \E\left( \left[\chosenstep\right]_i^2  \right) = \E\left( \Nlambda^2 \right) + n - 1  $. Therefore $\E\left(\|\chosenstep\|^2\right)/n - 1 = (\E \left( \Nlambda^2 \right) - 1)/n $.
By applying this to Eq.~\eqref{eq:stepsizechange}, we deduce that $\E( \ln(\sigma_{t+1}/\sigma_t) = 1/(2d_\sigma n)(\E(\Nmin^2) - 1)$. Furthermore, as $\E ( \Nlambda^2 ) \leq \E ((\lambda\NN)^2) = \lambda^2 < \infty$, we have $\E(\|\chosenstep\|^2) < \infty$. The sequence $(\|\chosenstep\|^2)_{t \in \N}$ being i.i.d according to Lemma~\ref{lem:selectedstep}, and being integrable as we just showed, we can apply the strong law of large numbers on Eq.~\eqref{eq:sigcumul}. We obtain
\begin{align*}
\inv{t}\ln \left({\sig \over \sigma_0} \right) &= \frac{1}{2d_\sigma}\left( \inv{t}\sum_{k=0}^{t-1}{\|\chosenstep[k]\|^2 \over n}  - 1  \right)\\
&\overset{a.s.}{\underset{t \rightarrow \infty}{\longrightarrow}} \frac{1}{2d_\sigma} \left( \frac{\E \left( \|\chosenstep[\cdot]\|^2 \right)}{n}   - 1   \right) = \frac{1}{2d_\sigma n} \left(  \E \left( \Nlambda^2 \right) - 1   \right) 
\end{align*}
~\\[-1.8\baselineskip]
\qed\end{proof}
The proposition reveals that the sign of $\left( \E \left( \Nlambda^2 \right) - 1 \right)$ determines whether the step-size diverges to infinity. In the following, we show that $\E \left( \Nlambda^2 \right)$ increases in $\lambda$ for $\lambda \geq 2$ and that the $(1,\lambda)$-ES diverges for $\lambda \geq 3$. For $\lambda=1$ and $\lambda=2$, the step-size follows a random walk on the log-scale. 

\begin{lemma} \label{lm:increasing_lambda_expectation}
	Let $(\mathcal{N}_i)_{i \in [[1, \lambda]]}$ be independent random variables, distributed according to $\NN$, and $\mathcal{N}_{i:\lambda}$ the $i^{th}$ order statistic of $(\mathcal{N}_i)_{i \in [[1, \lambda]]}$. Then $\E \left( \Nmin[1]^2 \right)=\E \left( \Nmin[2]^2 \right)= 1$. In addition, for all $\lambda \ge 2$, $\E \left( \Nmin[\lambda+1]^2 \right) > \E \left( \Nmin^2 \right)$. 
\end{lemma}

\begin{proof} (see \cite{chotard2012TRcumulative} for the full proof)
The idea of the proof is to use the symmetry of the normal distribution to show that for two random variables $U \sim \orderstatdis{\lambda+1}$ and $V \sim \orderstatdis{\lambda}$, for every event $E_1$ where $U^2 < V^2$, there exists another event $E_2$ counterbalancing the effect of $E_1$, i.e $\int_{E_2} (u^2-v^2)f_{U, V}(u,v) \diff \! u \diff \! v = \int_{E_1} (v^2-u^2)f_{U, V}(u,v) \diff \! u \diff \! v$, with $f_{U,V}$ the joint density of the couple $(U,V)$. We then have $\E \left( \Nmin[\lambda+1]^2 \right) \geq \E \left( \Nmin^2 \right)$. As there is a non-negligible set of events $E_3$, distinct of $E_1$ and $E_2$, where $U^2 > V^2$, we have $\E ( \Nmin[\lambda+1]^2 ) > \E ( \Nmin^2 )$.

For $\lambda = 1$, $\Nmin[1] \sim \NN$ so $\E(\Nmin[1]^2) = 1$. For $\lambda = 2$ we have $\E(\Nmin[2]^2+\mathcal{N}_{2:2}^2) = 2\E(\NN^2) = 2$, and since the normal distribution is symmetric $\E(\Nmin[2]^2) = \E(\mathcal{N}_{2:2}^2)$, hence $\E(\Nmin[2]^2) = 1$.
\qed\end{proof}
We can now link Proposition \ref{pr:sigrate} and Lemma \ref{lm:increasing_lambda_expectation} into the following theorem:

\begin{theorem} \label{th:convsig}
On linear functions, for $\lambda\ge3$, the step-size of the $(1, \lambda)$-CSA-ES without cumulation ($c=1$) diverges geometrically almost surely and in expectation at the rate $1/(2d_\sigma n)(\E(\Nmin^2) - 1)$, i.e.
\begin{equation}\label{eq:divas}
	\inv{t} \ln\left(\frac{\sig}{\sig[0]}\right) 
	\;
	\overset{a.s.}{\underset{t \rightarrow \infty}{\longrightarrow}} 
	\;
	\E\left(\ln\left(\frac{\sigma_{t+1}}{\sig}\right)\right)
	\;=\; 
	\frac{1}{2d_{\sigma}n} \left( \E \left( \Nlambda^2 \right) - 1 \right) 
	\enspace.
	\end{equation}

For $\lambda=1$ and $\lambda=2$, without cumulation, the logarithm of the step-size does an additive unbiased random walk i.e.
$\ln \sigma_{t+1} = \ln \sigma_{t} + W_{t} $ where $E[W_{t}]=0$. More precisely $W_{t} \sim 1/(2d_\sigma)(\chi_n^2/n - 1)$ for $\lambda=1$, and $W_{t} \sim 1/(2d_\sigma)((\Nmin[2]^2 + \chi_{n-1}^2)/n - 1)$ for $\lambda =2$, where $\chi_k^2$ stands for the chi-squared distribution with $k$ degree of freedom.
\end{theorem}
\begin{proof}
For $\lambda > 2$, from Lemma~\ref{lm:increasing_lambda_expectation} we know that $\E(\Nmin^2) > \E(\Nmin[2]^2) = 1$. Therefore $\E(\Nmin^2) - 1 > 0$, hence Eq.~\eqref{eq:divas} is strictly positive, and with Proposition~\ref{pr:sigrate} we get that the step-size diverges geometrically almost surely at the rate $1/(2d_\sigma)(\E(\Nmin^2) - 1)$. 

With Eq.~\ref{eq:stepsizechange} we have $\ln(\sigma_{t+1}) = \ln(\sigma_t) + W_t$, with $W_t = 1/(2d_\sigma)(\|\chosenstep\|^2/n - 1)$. For $\lambda = 1$ and $\lambda = 2$, according to Lemma~\ref{lm:increasing_lambda_expectation}, $\E(W_t) = 0$. Hence $\ln(\sigma_t)$ does an additive unbiased random walk. Furthermore $\|\chosenstepdistribution\|^2 = \Nmin^2 + \chi_{n-1}^2$, so for $\lambda = 1$, since $\Nmin[1] = \NN$, $\|\chosenstepdistribution\|^2 = \chi_n^2$.
\qed\end{proof}

In \cite{chotard2012TRcumulative} we extend this result on the step-size to $|[\x]_1|$, which diverges geometrically almost surely at the same rate.

\section{Divergence rate of $(1, \lambda)$-CSA-ES with cumulation} \label{sec:withcumulation}

We are now investigating the $(1,\lambda)$-CSA-ES with cumulation, i.e. $0 < c < 1$.
The path $\pchain$ is then a Markov chain and contrary to the case where $c=1$ we cannot apply a LLN for independent variables to Eq.~\eqref{eq:sigcumul} in order to prove the almost sure geometric divergence. However LLN for Markov chains exist as well, provided the Markov chain satisfies some stability properties: in particular, if the Markov chain $\pchain$ is $\varphi$-irreducible, that is, there exists a measure $\varphi$ such that every Borel set  $A$ of $\R^{n}$  with $\varphi(A)>0$ has a positive probability to be reached in a finite number of steps  by $\pchain$ starting from any $\bs{p}_{0} \in \R^{n}$. In addition, the chain $\pchain$ needs to be (i) positive, that is the chain admits an invariant probability measure $\pi$, i.e., for any borelian $A$, $\pi(A) = \int_{\R^n} P(x, A) \pi(A)$ with $P(x,A)$ being the probability to transition in one time step from $x$ into $A$, and (ii) Harris recurrent which means for any borelian $A$ such that $\varphi(A) > 0$, the chain $\pchain$ visits $A$ an infinite number of times with probability one. Under those conditions, $\pchain$ satisfies a LLN, more precisely:
\begin{lemma}\cite[17.0.1]{markovtheory} \label{llnmarkovchains}
Suppose that $\pchain$ is a positive Harris chain with invariant probability measure $\pi$, and let $g$ be a $\pi$-integrable function such that \\ $\pi(|g|) = \int_{\R^n} |g(x)| \pi(dx) < \infty$. Then $1/t \sum_{k=1}^t g(\p[k]) \overset{a.s}{\underset{t\rightarrow \infty}{\longrightarrow}} \pi(g)$.
\end{lemma}
The path $\pchain$ satisfies the conditions of Lemma~\ref{llnmarkovchains} and exhibits an invariant measure \cite{chotard2012TRcumulative}.
By a recurrence on Eq.~\eqref{eq:chemin} we see that the path follows the following equation
\begin{equation} \label{eq:path}
\p = (1-c)^{t}\bs{p}_0 + \sqrt{c(2-c)}\sum_{k=0}^{t-1} (1-c)^k \underbrace{\chosenstep[t-1-k]}_{\text{i.i.d.}} \enspace .
\end{equation}
For $i \neq 1$, $[\chosenstep]_i \sim \NNN(0, 1)$ and,
as also $[\bs{p}_0]_i \sim \NNN(0,1)$,
% $[\bs{p_1}]_i \sim \NNN(0,1)$ too. 
by recurrence $[\bs{p_t}]_i \sim \NNN(0,1)$ for all $t \in \N$. 
For $i=1$ with cumulation ($c <1$), the influence of $[\bs{p}_0]_1$ vanishes with $(1-c)^{t}$. Furthermore, as from Lemma~\ref{lem:selectedstep} the sequence $(\fchosenstep])_{t \in \N}$ is independent, we get by applying the Kolgomorov's three series theorem that the series $\sum_{k=0}^{t-1} (1-c)^k \fchosenstep[t-1-k]$ converges almost surely. Therefore, the first component of the path becomes distributed as the random variable $[\bs{p}_\infty]_1 = \sqrt{c(2-c)} \sum_{k=0}^\infty (1-c)^k [\chosenstep[k]]_1$ (by re-indexing the variable $\bs{\xi}_{t-1-k}^\star$ in $\bs{\xi}_k^\star$, as the sequence $(\chosenstep)_{t \in \N}$ is i.i.d.). 
%, and for $i \neq 1$, $[\bs{p}_\infty]_i \sim \NNN(0,1)$.

We now obtain geometric divergence of the step-size and get an explicit estimate of the expression of the divergence rate.
\begin{theorem} \label{th:geometricdivergencecumul} 
The step-size of the $(1,\lambda)$-CSA-ES with $\lambda\ge2$ diverges geometrically fast if $c<1$ or $\lambda\ge3$. Almost surely and in expectation we have for $0<c\le1$,
\begin{equation}\label{eq:sigcumulconv}\inv{t} \ln\left( {\sig \over \sig[0]} \right) \overset{}{\underset{t \rightarrow \infty}{\longrightarrow}} 
\frac{1}{2d_\sigma n} 
\underbrace{
%\left( \E\left(\Nlambda^2\right) +  \frac{2-2c}{c}\E\left(\Nlambda\right)^2 - c \right)
\left( 2(1-c)\,\E\left(\Nlambda\right)^2  + c\left(\E\left(\Nlambda^2\right) - 1\right) \right)
}_{>0 \mbox{ for } \lambda\ge3 \mbox{ and for } \lambda = 2 \mbox{ and } c < 1 } 
\enspace.
\end{equation}

\end{theorem}
\begin{proof}
For proving almost sure convergence of $\ln( {\sig / \sig[0]}) / t$ we need to use the LLN for Markov chain. We refer to \cite{chotard2012TRcumulative} for the proof that $\pchain$ satisfies the right assumptions. We now focus on the convergence in expectation.
From Eq.~\eqref{eq:stepsizechange} we have $\E (\ln(\sigma_{t+1} / \sigma_t)) = c/(2d_\sigma)(\E(\|\pplus\|^2)/n - 1)$, so $\E(\|\pplus\|^2) =  \E(\sum_{i=1}^n \pdim[i]^2) $ is the term we have to analyse. From Eq.~\eqref{eq:path} and its conclusions we get that for $\j \neq 1$ $[\bs{p}_t]_j \sim \NN$, so $\E(\sum_{j=1}^n \pdim[j]^2) = \E(\pdim[1]^2) + (n-1)$. When $t$ goes to infinity, the influence of $[\bs{p}_0]_1$ in this equation goes to $0$ with $(1-c)^{t+1}$, so we can remove it when taking the limit:
\begin{equation}
\lim_{t \rightarrow \infty}\E\left( \pdim[1]^2 \right) = \lim_{t \rightarrow \infty} \E\bigg( \bigg(\sqrt{c(2-c)}  \sum_{i=0}^t (1-c)^i \firstdim{\chosenstep[t-i]} \bigg)^2\bigg)
\end{equation}
We will now develop the sum with the square, such that we have either a product $\firstdim{\chosenstep[t-i]} \firstdim{\chosenstep[t-j]}$ with $i \neq j$, or $\firstdim{\chosenstep[t-j]}^2$. This way, we can separate the variables by using Lemma~\ref{lem:selectedstep} with the independence of $\chosenstep[i]$ over time. To do so, we use the development formula $(\sum_{i=1}^n a_n)^2 = 2 \sum_{i=1}^n \sum_{j=i+1}^n a_i a_j + \sum_{i=1}^n a_i^2$. We take the limit of $\E( \pdim[1]^2 )$ and find that it is equal to
\begin{equation} \label{eq:devexppathsquare}
 \lim_{t \rightarrow \infty}c(2-c) \! \left( 2\sum_{i=0}^t \! \sum_{j=i+1}^t\!\! (1 \! - \! c)^{i+j} \!\!\!\! \underbrace{\E\left( \firstdim{\chosenstep[t-i]} \firstdim{\chosenstep[t-j]} \right)}_{=\E \firstdim{\chosenstep[t-i]} \E \firstdim{\chosenstep[t-j]}=\E [\Nmin]^{2}} \!\!\!\! + \sum_{i=0}^t (1 \! - \! c)^{2i} \underbrace{\E\left(\firstdim{\chosenstep[t-i]}^2\right)}_{=\E[\Nmin^2]} \right)
\end{equation}
Now the expected value does not depend on $i$ or $j$, so what is left is to calculate $\sum_{i=0}^t \sum_{j=i+1}^t (1-c)^{i+j}$ and $\sum_{i=0}^t (1-c)^{2i}$. We have $\sum_{i=0}^t \sum_{j=i+1}^t (1-c)^{i+j} = \sum_{i=0}^t (1-c)^{2i+1} \frac{1-(1-c)^{t-i}}{1-(1-c)}$ and when we separates this sum in two, the right hand side goes to $0$ for $t\to\infty$. Therefore, the left hand side converges to $\lim_{t \rightarrow \infty} \sum_{i=0}^t (1-c)^{2i+1}/c$, which is equal to $\lim_{t \rightarrow \infty} (1-c)/c \sum_{i=0}^t (1-c)^{2i}$. And $\sum_{i=0}^t (1-c)^{2i}$ is equal to $(1-(1-c)^{2t+2})/(1-(1-c)^2)$, which converges to $1/(c(2-c))$. So, by inserting this in Eq.~\eqref{eq:devexppathsquare} we get that $\E\left( \pdim[1]^2 \right) \underset{t \rightarrow \infty}{\longrightarrow} 2\frac{1-c}{c} \E\left(\Nmin\right)^2 + \E\left(\Nmin^2\right)$, which gives us the right hand side of Eq.~\eqref{eq:sigcumulconv}.

By summing $\E(\ln (\sigma_{i+1}/\sigma_i))$ for $i=0,\dots, t-1$ and dividing by $t$ we have the Cesaro mean $1/t \E(\ln (\sigma_{t}/\sigma_0))$ that converges to the same value that $\E(\ln (\sigma_{t+1}/\sigma_t))$ converges to when $t$ goes to infinity. Therefore we have in expectation Eq.~\eqref{eq:sigcumulconv}.

According to Lemma~\ref{lm:increasing_lambda_expectation}, for $\lambda = 2$, $\E(\Nmin[2]^2) = 1$, so the RHS of Eq.~\eqref{eq:sigcumulconv} is equal to $(1-c)/(d_\sigma n) \E(\Nmin[2])^2$. The expected value of $\Nmin[2]$ is strictly negative, so the previous expression is strictly positive. Furthermore, according to Lemma~\ref{lm:increasing_lambda_expectation}, $\E(\Nmin^2)$ increases with $\lambda$, as does $\E(\Nmin[2])^2$. Therefore we have geometric divergence for $\lambda \geq 2$.
\qed\end{proof}

From Eq.~\eqref{eq:selection2} we see that the behavior of the step-size and of $\suite{\bs{X}}$ are directly related. Geometric divergence of the step-size, as shown in Theorem~\ref{th:geometricdivergencecumul}, means that also the movements in search space and the improvements on affine linear functions $f$ increase geometrically fast. Therefore, as we showed in Theorem~\ref{th:geometricdivergencecumul} geometric divergence for the step-size when $\lambda \geq 2$ and $c < 1$, or when $\lambda \geq 3$, we expect geometric divergence on the first dimension of $\suite{\bs{X}}$ (the first dimension being the only dimension with selection pressure). Analyzing $\suite{\bs{X}}$ with cumulation requires to study a double Markov chain, which is left to possible future research.

\section{Study of the variations of $\ln\left({\sig[t+1]}/{\sig}\right)$} \label{sec:var}
The proof of Theorem~\ref{th:geometricdivergencecumul} shows that the step size increase converges to the right hand side of Eq.~\eqref{eq:sigcumulconv}, for $t\to\infty$. When the dimension increases this increment goes to zero, which also suggests that it becomes more likely that $\sigplus$ is smaller than $\sig$. To analyze this behavior, we study the variance of $\ln\left({\sig[t+1]}/{\sig}\right)$ as a function of $c$ and the dimension.

\begin{theorem} \label{th:var} The variance of $\ln\left({\sig[t+1]}/{\sig}\right)$ equals to
\begin{equation} \label{eq:var} \Var \left(\ln\left(\frac{\sigma_{t+1}}{\sigma_{t}}\right)\right) =
\frac{c^2}{4 d_\sigma^2 n^2}\left( \E\left(\firstdim{\pplus}^4 \right) - \E\left(\firstdim{\pplus}^2\right)^2 + 2(n-1) \right) 
\enspace.
\end{equation}
Furthermore, $\E\left(\pdim[1]^2\right) \underset{t \rightarrow \infty}{\longrightarrow} \E\left(\Nmin^2\right) + \frac{2-2c}{c}\E\left(\Nmin\right)^2$ and with $a = 1-c$ 
\begin{equation} \label{eq:devp4}
\lim_{t \to \infty}\E \left(\pdim[1]^4\right)= \frac{(1-a^2)^2}{1-a^4}\left( k_4 +k_{31} + k_{22} + k_{211} + k_{1111} \right) 
\enspace,
\end{equation}
where $k_{4}\!=\!\E\!\left(\Nmin^4\right)$, $k_{31}=  4\frac{a\left(1+a+2a^2\right)}{1-a^3}\E\left(\Nmin^3\right)\E\left(\Nmin\right)$, $k_{22}= 6\frac{a^2}{1-a^2}\E\left( \Nmin^2\right)^2$,\\ $k_{211}\!=\!12\frac{a^3(1+2a+3a^2)}{(1-a^2)(1-a^3)} \E\!\left(\Nmin^2\right)\! \E\!\left(\Nmin\right)^2$ and $k_{1111} = 24\frac{a^6}{(1-a)(1-a^2)(1-a^3)}\E\left(\Nmin\right)^4$.
\end{theorem}

\begin{proof} 
\begin{equation} \label{eq:proofvar}
\Var \left( \ln\left( \frac{\sigma_{t+1}}{\sigma_t} \right) \right) = \Var \left( \frac{c}{2d_\sigma}\left( \frac{\|\pplus \|^2}{n} - 1\right)\right) = \frac{c^2}{4d_\sigma^2 n^2}\!\!\!\! \underbrace{\Var\left(\|\pplus\|^2 \right)}_{\E\left(\|\pplus\|^4\right) - \E\left(\|\pplus\|^2 \right)^2 }
\end{equation}
The first part of $\Var(\|\pplus\|^2)$, $\E(\|\pplus\|^4)$, is equal to $\E((\sum_{i=1}^n \pdim^2 )^2 )$. We develop it along the dimensions such that we can use the independence of $[\pplus]_i$ with $[\pplus]_j$ for $i\neq j$, to get $\E( 2 \sum_{i=1}^n \sum_{j=i+1}^n \pdim[i]^2\pdim[j]^2  + \sum_{i= 1}^n \pdim[i]^4 )$. For $i \neq 1$ $\pdim[i]$ is distributed according to a standard normal distribution, so $\E\left(\pdim[i]^2\right) = 1$ and $\E\left(\pdim[i]^4\right) = 3$.
\begin{align*}
\E\left(\|\pplus\|^4\right) &= 2 \sum_{i=1}^n \sum_{j=i+1}^n \E\left(\pdim[i]^2\right) \E\left(\pdim[j]^2\right)  + \sum_{i=1}^n \E\left(\pdim[i]^4\right) \\
&= \left(2 \sum_{i=2}^n \sum_{j=i+1}^n 1\right) + 2\sum_{j=2}^n \E\left(\pdim[1]^2\right) + \left(\sum_{i=2}^n 3\right) + \E\left(\pdim[1]^4 \right)\\
&= \left( 2 \sum_{i=2}^n (n-i) \right) + 2(n-1) \E \left(\pdim[1]^2\right) + 3(n-1) + \E\left(\pdim[1]^4 \right)\\
&= \E\left(\pdim[1]^4 \right) + 2(n-1)\E \left(\pdim[1]^2\right) + (n-1)(n+1)
%
%\Var \left( \ln\left( \frac{\sigma_{t+1}}{\sigma_t} \right) \right) &= \frac{1}{4d_\sigma^2 n^2} \left(\E\left(\|\chosenstep\|^4\right) - \E\left(\|\chosenstep\|^2 \right)^2 \right)
\end{align*}
The other part left is $\E(\|\pplus\|^2 )^2$, which we develop along the dimensions to get $\E ( \sum_{i=1}^n \pdim^2  )^2 = ( \E(\pdim[1]^2) + (n-1) )^2$, which equals to $\E(\pdim[1]^2)^2 + 2(n-1)\E(\pdim[1]^2) + (n-1)^2$.
So by subtracting both parts we get \\$\E(\|\pplus\|^4) - \E(\|\pplus\|^2 )^2 = \E(\pdim[1]^4 )  - \E(\pdim[1]^2)^2 + 2(n-1)$, which we insert into Eq.~\eqref{eq:proofvar} to get Eq.~\eqref{eq:var}.

The development of $\E(\pdim[1]^2)$ is the same than the one done in the proof of Theorem~\ref{th:geometricdivergencecumul}. We refer to \cite{chotard2012TRcumulative} for the development of $\E(\pdim[1]^4)$, since limits of space in the paper prevents us to present it here.
\qed\end{proof}

\begin{figure}[t,b]
\centering\vspace*{-1ex}
\subfloat[Without cumulation ($c=1$)]{\label{fig:simwithoutcumulation}\includegraphics[width=0.5\textwidth,trim=0 0 0 12mm,clip]{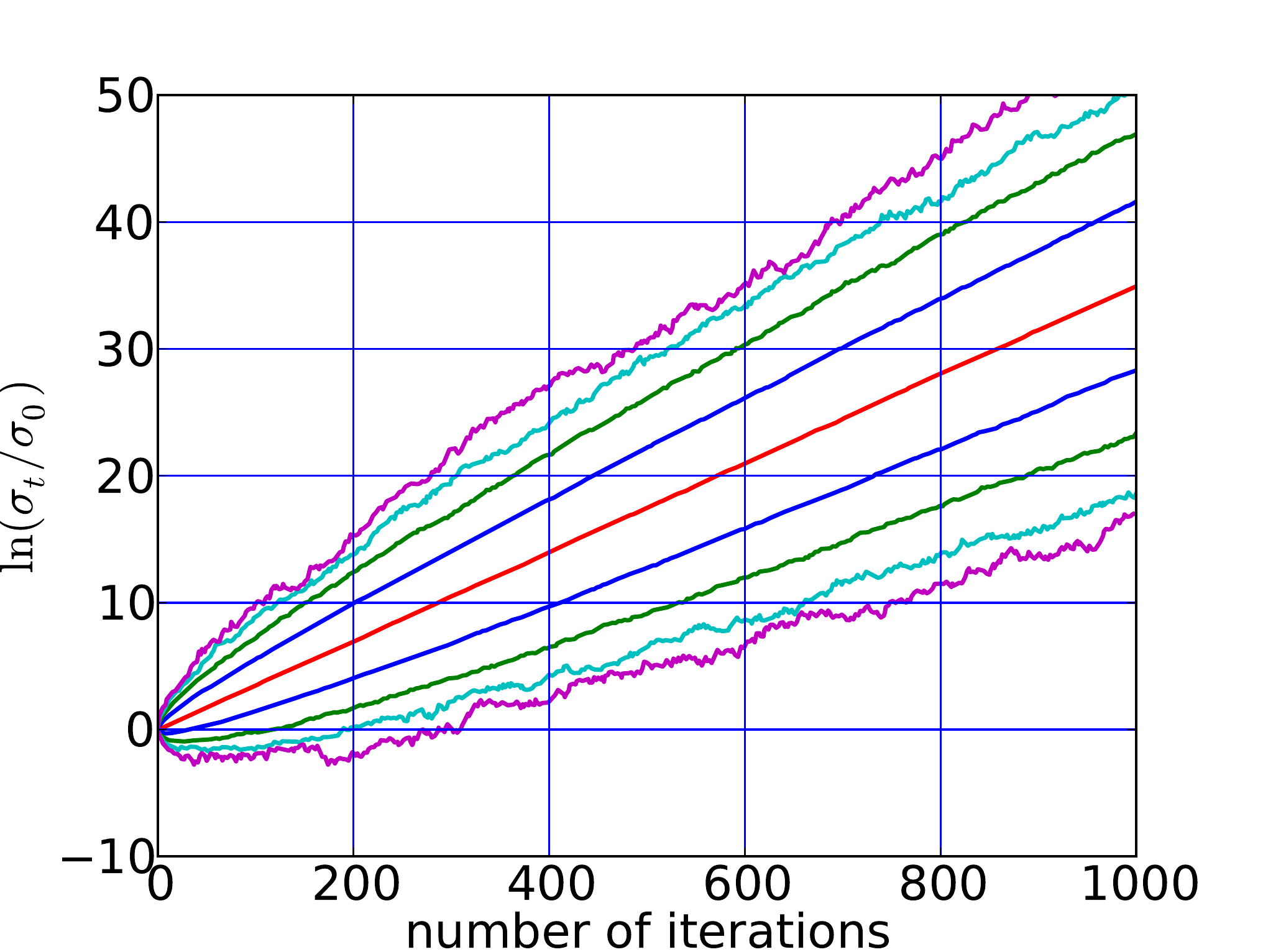}}
\subfloat[With cumulation ($c = 1/\sqrt{20}$)]{\label{fig:simwithcumulation}\includegraphics[width=0.5\textwidth,trim=0 0 0 12mm,clip]{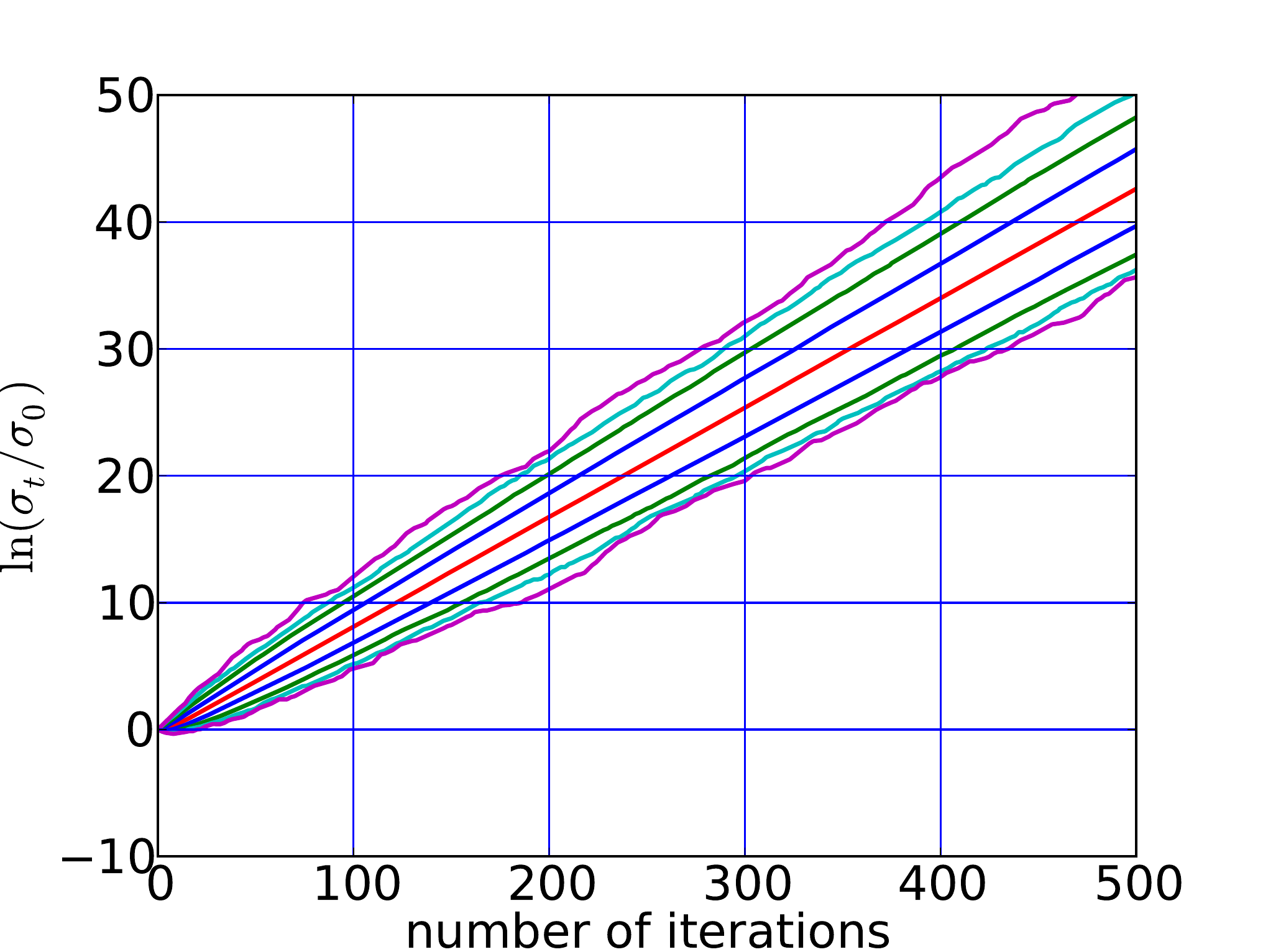}}
\caption{\label{fig:logsigmaevolution}
$\ln(\sig / \sigma_0)$ against $t$. The different curves represent the quantiles of a set of $5.10^3+1$ samples, more precisely the $10^{i}$-quantile and the $1 - 10^{-i}$-quantile for $i$ from $1$ to $4$; and the median. We have $n=20$ and $\lambda = 8$. }
\end{figure}

Figure~\ref{fig:logsigmaevolution} shows the time evolution of $\ln(\sigma_t / \sigma_0)$ for 5001 runs and $c=1$ (left) and $c=1/\sqrt{n}$ (right).
By comparing Figure~\ref{fig:simwithoutcumulation} and Figure~\ref{fig:simwithcumulation} we observe smaller variations of $\ln(\sigma_t / \sigma_0)$ with the smaller value of $c$.

\begin{figure}
%\centering
%For $c$ from $10^{-2}$ to $1$, and $n=2$, $20$, $200$, $2000$]
\includegraphics[width=0.49\textwidth]{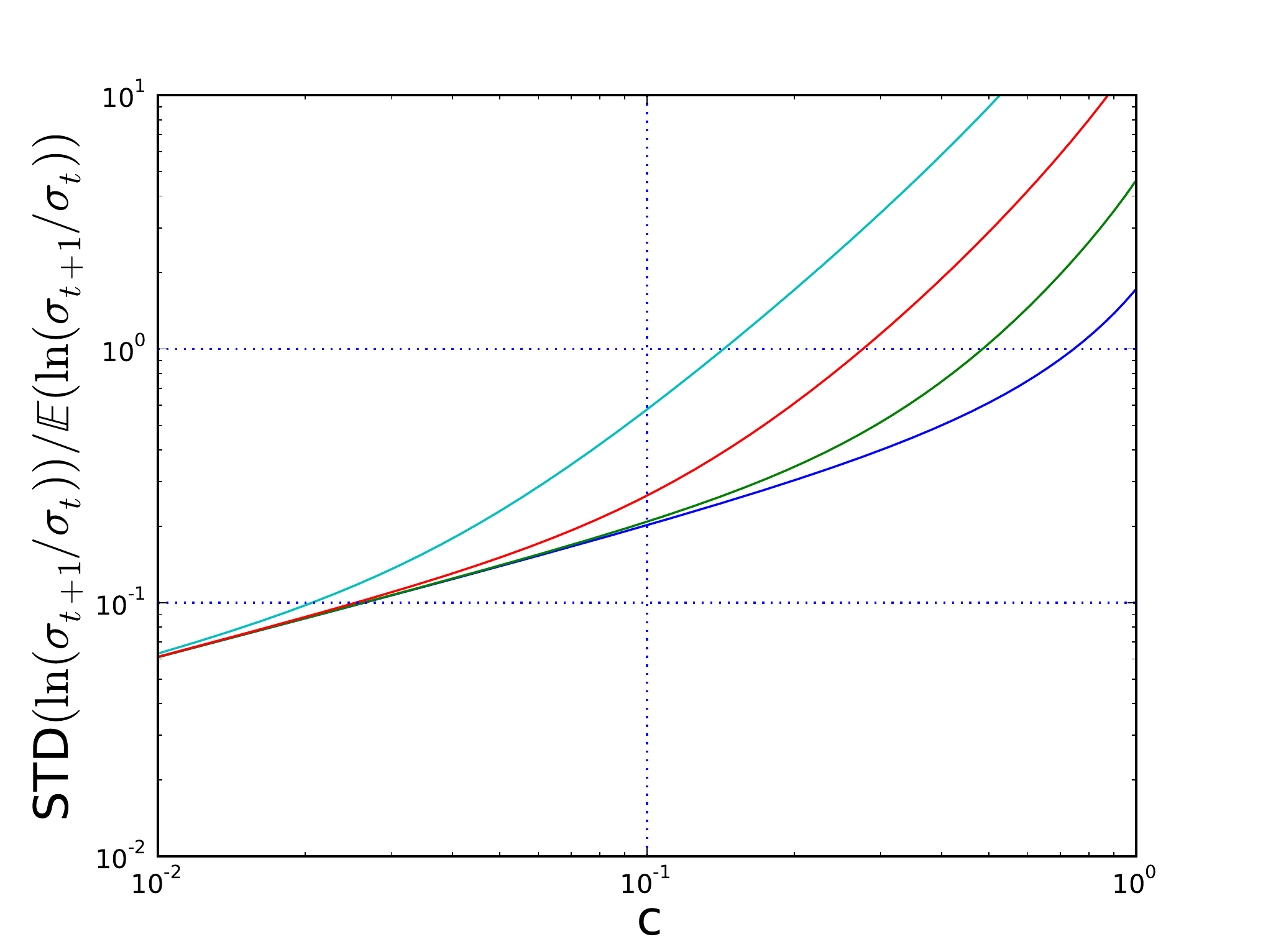}
%\caption{Standard deviation of $\ln\left(\sig\right)$ relatively to its expectation, for $c$ from $0.01$ to $1$, $\lambda = 8$, and $n = 3$.}
%For $n$ from $1$ to $10^4$, and $c = 1.$, $0.5$, $0.2$, $1/(1+n^{1/4})$, $1/(1+n^{1/3})$, $1/(1+n^{1/2})$ and $1/(1+n)$
\includegraphics[width=0.49\textwidth]{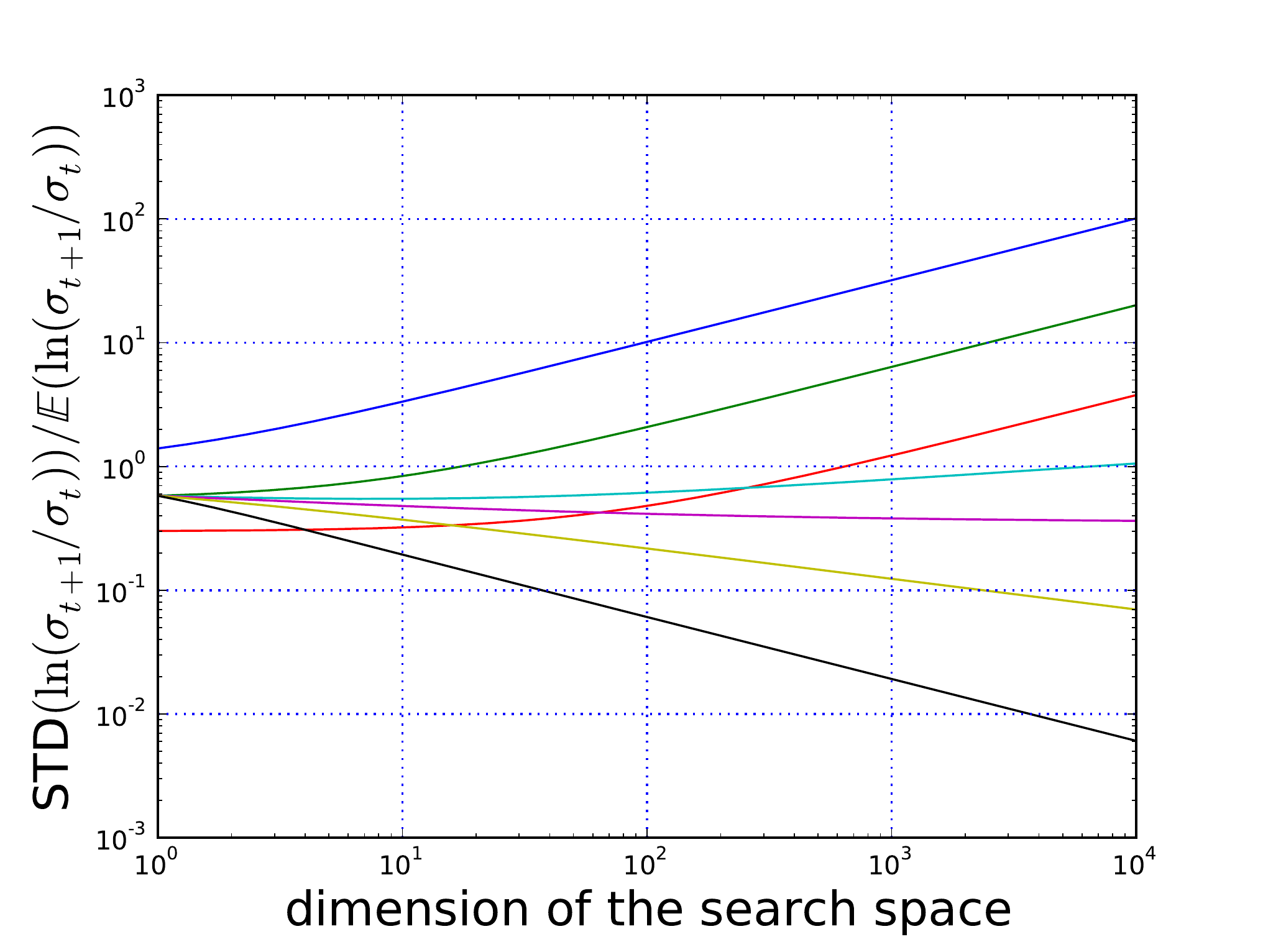}
%\caption{Standard deviation of $\ln\left(\sig\right)$ relatively to its expectation, for $n$ from $1$ to $2000$, $\lambda = 8$, and $c = 0.5$.}
\caption{\label{fig:stdandc}
Standard deviation of $\ln\left(\sigma_{t+1}/\sig\right)$ relatively to its expectation. Here $\lambda = 8$. The curves were plotted using Eq.~\eqref{eq:var} and Eq.~\eqref{eq:devp4}. On the left, curves for (right to left) $n=2$, $20$, $200$ and $2000$. On the right, different curves for (top to bottom) $c = 1$, $0.5$, $0.2$, $1/(1+n^{1/4})$, $1/(1+n^{1/3})$, $1/(1+n^{1/2})$ and $1/(1+n)$.}% Not sure about this last sentence: We believe it quite acurate because the distribution of $\p$ converges fast to its invariant measure. }
\end{figure}

Figure~\ref{fig:stdandc} shows the relative standard deviation of $\ln\left(\sigma_{t+1}/ \sigma_t \right)$ (i.e. the standard deviation divided by its expected value). 
Lowering $c$, as shown in the left, decreases the relative standard deviation. To get a value below one, $c$ must be smaller for larger dimension. In agreement with Theorem \ref{th:var}, In Figure~\ref{fig:stdandc}, right, the relative standard deviation increases like $\sqrt{n}$ with the dimension for constant $c$ (three increasing curves). A careful study \cite{chotard2012TRcumulative} of the variance equation of Theorem \ref{th:var} shows that for the choice of $c = 1/(1+n^{\alpha})$, if $\alpha > 1/3$ the relative standard deviation converges to $0$ with $\sqrt{(n^{2\alpha} + n)/n^{3 \alpha}}$. Taking $\alpha = 1/3$ is  a critical value where the relative standard deviation converges to $1/(\sqrt{2} \E(\Nlambda)^2)$. On the other hand, lower values of $\alpha$ makes the relative standard deviation diverge with $n^{(1-3\alpha)/2}$.

\section{Summary} \label{sec:conclusion}

We investigate throughout this paper the ($1,\lambda$)-CSA-ES on affine linear functions composed with strictly increasing transformations. We find, in Theorem~\ref{th:geometricdivergencecumul}, the limit distribution for $\ln(\sig/\sigma_0)/t$ and rigorously prove the desired behaviour of $\sigma$ with $\lambda\ge3$~ for any $c$, and with $\lambda=2$ and cumulation ($0<c<1$): the step-size diverges geometrically fast. In contrast, without cumulation ($c=1$) and with $\lambda=2$, a random walk on $\ln (\sigma)$ occurs, like for the ($1,2$)-$\sigma$SA-ES \cite{hansen2006ecj} (and also for the same symmetry reason). We derive an expression for the variance of the step-size increment. On linear functions when $c = 1/n^{\alpha}$, for $\alpha \geq 0$ ($\alpha = 0$ meaning $c$ constant) and for $n \to \infty$ the standard deviation is about $\sqrt{(n^{2\alpha} + n)/n^{3 \alpha}}$ times larger than the step-size increment. From this follows that keeping $c < 1/n^{1/3}$ ensures that the standard deviation of $\ln(\sigma_{t+1}/ \sigma_t)$ becomes negligible compared to $\ln(\sigma_{t+1}/ \sigma_t)$ when the dimensions goes to infinity. That means, the signal to noise ratio goes to zero, giving the algorithm strong stability. The result confirms that even the largest default cumulation parameter $c=1/\sqrt{n}$ is a stable choice.

\section*{Acknowledgments}
This work was partially supported by the ANR-2010-COSI-002
grant (SIMINOLE) of the French National Research Agency and the ANR COSINUS project ANR-08-COSI-007-12.

\bibliography{biblio}

\bibliographystyle{plain}

\end{document}